\documentclass[letterpaper, 10 pt, conference]{ieeeconf}  

\IEEEoverridecommandlockouts                              

\usepackage{graphics} 
\usepackage{url}
\usepackage{booktabs}
\usepackage{theorem}
\usepackage{amsmath, amsfonts} 
\usepackage{amssymb}  
\usepackage{mathtools}
\usepackage{multicol}
\usepackage{graphicx}
\usepackage{layouts}
\usepackage{wrapfig}
\usepackage{graphicx}
\usepackage{easyReview}
\usepackage{printlen}
\usepackage{bm}
\usepackage{siunitx}
\usepackage{colortbl}	
\usepackage[utf8]{inputenc}
\usepackage{pgf}
\usepackage{tikz}
\usepackage{tikzscale}
\usepackage{pgfplots}
\usepackage{array}
\DeclareUnicodeCharacter{2212}{\textendash}
\pgfplotsset{compat=1.14}
\usepgfplotslibrary{external}
\usepgfplotslibrary[external]
\tikzset{external/force remake}
\usepackage[vlined,ruled,linesnumbered,procnumbered]{algorithm2e}
\usepackage{algpseudocode}
\usepackage{makecell}
\usepackage{hyperref}
\usepackage[dvipsnames]{xcolor}
\usepackage[caption=false,font=footnotesize]{subfig}
\usepackage{footnote}
\makesavenoteenv{tabular}
\newtheorem{remark}{\textbf{Remark}}

\newtheorem{theorem}{Theorem}

\usepackage{multirow}

\def\*#1{\mathbf{#1}}
\def\b#1{\boldsymbol{#1}}

\newcommand{\rev}[1]{\textcolor{black}{{#1}}}
\title{\LARGE \bf
Hippo: 
\underline{H}igh-performance \underline{I}nterior-\underline{P}oint and \underline{P}rojection-based Solver for Generic Constrained Trajectory \underline{O}ptimization

\author{
Haizhou Zhao$^{1, 2}$
 , Ludovic Righetti$^{1}$
 , Majid Khadiv$^{2}$
}
\thanks{$^{1}$Tandon School of Engineering, New York University, USA, {\tt\small haizhou.zhao@nyu.edu,ludovic.righetti@nyu.edu}\
$^{2}$Munich Institute of Robotics and Machine Intelligence (MIRMI), Technical University of Munich (TUM), Germany {\tt\small majid.khadiv@tum.de}
}
}

\begin{document}

\maketitle
\thispagestyle{empty}
\pagestyle{empty}


\begin{abstract}
	Trajectory optimization is the core of modern model-based robotic control and motion planning. Existing trajectory optimizers, based on sequential quadratic programming (SQP) or differential dynamic programming (DDP), are often limited by their slow computation efficiency, low modeling flexibility, and poor convergence for complex tasks requiring hard constraints. In this paper, we introduce \textit{Hippo}\footnote{Github link: https://github.com/HaizhouZ/moto}, a solver that can handle inequality constraints using the interior-point method (IPM) with an adaptive barrier update strategy and hard equality constraints via projection or IPM. Through extensive numerical benchmarks, we show that \textit{Hippo} is a robust and efficient alternative to existing state-of-the-art solvers for difficult robotic trajectory optimization problems requiring high-quality solutions, such as locomotion and manipulation.
\end{abstract}

\section{Introduction}
Direct multiple shooting (DMS) plays an important role in modern robotic trajectory optimization due to its superior robustness and faster convergence \cite{liftednewton}. Several approaches enable faster resolution of the problem by exploiting the temporal sparsity structure in optimal control problems (OCP), using Riccati recursion \cite{mastalli20crocoddyl, ocs2constr, proxddp} or block elimination \cite{stagewisesqp, parallelandproximalLQR}, which are fundamentally equivalent \cite{stagewisesqp}. Although many of them are successfully deployed for Model Predictive Control (MPC), they often fail to solve large-scale or difficult OCPs due to poor convergence and slow computation. These limitations hinder their reliable application to data generation for behavior cloning (BC) or task and motion planning (TAMP). In this paper, we propose a trajectory optimizer designed to address these issues.

\subsection{Related Works\label{sec:relatedwork}}

Popular constraint handling techniques are the interior point method (IPM) \cite{Verschueren2021acados, sotaronewtonmethod,interiorpointddp} and the augmented Lagrangian method (ALM) \cite{ocs2constr,constrainedddprevisited, proxddp,stagewisesqp,parallelandproximalLQR}. Both have homotopy parameters (i.e., barrier parameters for IPM and penalty parameters for ALM) iteratively reduced to recover the solution to the original OCP. 

IPM converts inequality constraints to slacked equality constraints and relaxes complementary slackness in KKT conditions with an evolving barrier parameter. \textit{CALIPSO} \cite{calipso} and \textit{robotoc} \cite{sotaronewtonmethod} implement a monotonic update rule of IPM, which reduces the barrier parameter when the KKT residual is lower than a threshold. However, it can be sensitive to the choice of initial barrier parameter \cite{ipopt_adaptive_barrier}. \textit{acados} \cite{Verschueren2021acados} is an SQP solver based on \textit{HPIPM} \cite{hpipm}, an IPM-based QP solver with Mehrotra predictor-corrector (MP-C) adaptive barrier update strategy \cite{mehrotra-pc}. It represents equality constraints as box-constrained inequalities with zero boundaries. In each SQP iteration, it calls HPIPM to solve the stagewise linear-quadratic (LQ) subproblem to an acceptable accuracy. Recently, \textit{piqp} \cite{schwan2023piqp} combined the proximal method of multipliers (similar to ALM) and IPM using MP-C, and achieved state-of-the-art performance in \textit{qpbenchmark} \cite{qpbenchmark}.                                       

ALM transcribes constrained optimization to an unconstrained one by penalizing constraint violation. Specifically, \textit{aligator} \cite{proxddp} and \textit{mim\_solver} \cite{stagewisesqp} employ an ALM-like algorithm in a semi-smooth Newton form, a variant of the primal-dual active-set strategy (PDAS) \cite{pdas}. The main problem with ALM is its poor convergence in solving inequality-constrained OCPs due to the explicit detection of active sets \cite{interiorpointddp}, and it can easily get trapped in local minima. In contrast, IPM enjoys polynomial convergence complexity thanks to the self-concordance of the barrier design \cite{ipm_polynomial}. Another issue with ALMs is the penalty parameter scheduling scheme, which remains an open question \cite{osqp}. On the contrary, IPM has systematic adaptive barrier update strategies such as MP-C and quality functions \cite{ipopt_adaptive_barrier}.

The two-stage exact minimization framework of \textit{acados} and \textit{mim\_solver} approximately solves the LQ subproblem in inner loops. For ALM-like methods, this can sometimes be helpful since it increases the chance of finding active sets using cheap recursion over only the value function Jacobians. For IPM, such inner loops can be computationally expensive since the Hessian factorization cannot be reused due to the varying slack and multipliers. Moreover, for nonlinear optimization, such a framework introduces additional solver parameters to control the termination of the inner loops, such as the subproblem tolerances. These parameters can be highly problem-specific and require careful hand-tuning. 

IPM and ALM share common sparsity in their stagewise KKT system and thus can be solved in a similar routine: during the recursion, only the primal steps are computed; later in the post-recursion stage, dual steps are reconstructed in parallel \cite{sotaronewtonmethod, stagewisesqp}. In \cite{proxddp}, it is claimed that this splitting suffers from poor numerical conditioning due to inversion of diagonal coefficients, which may reach extreme values upon quasi-convergence. Instead, \cite{proxddp} chose to use Cholesky LDLT factorization suitable for semi-positive-definite KKT linear systems, at the cost of slower computation due to the high complexity of LDLT and matrix-matrix multiplication with higher dimension. 
To further speed up the optimization, \textit{ocs2} \cite{Farshidian2017RealtimeMP} parallelizes the sequential recursion by breaking the causal temporal structure of OCP into several segments so that recursion over each segment can start with the information from the last iteration. This corresponds to the classical \textit{parallel-in-time} integration to solve initial value problems \cite{parareal}. Despite the speed improvement, this inconsistency can slow down convergence or even lead to divergence. \textit{aligator} \cite{parallelandproximalLQR} splits the trajectory using the segmental sensitivities w.r.t. the co-states, i.e., the dynamics multipliers at the breakpoints. Although its solution is exact and will not affect convergence, the computation of extra sensitivities and factorization render the speed improvements less promising.

Apart from specialized OCP solvers, \textit{IPOPT} \cite{ipopt_implementation} is recognized as a generic robust nonlinear programming (NLP) solver and widely used to solve complex multi-contact trajectory optimization problems \cite{taouil2025physicallyconsistenthumanoidlocomanipulation,jin2025complementarityfreemulticontactmodelingoptimization, ciebielski2025taskmotionplanninghumanoid}. In contrast with \textit{acados} that requires two-stage exact minimization, \textit{IPOPT} computes the primal-dual Newton step at each SQP iteration with powerful globalization, including feasibility restoration phases and filter line search. Despite its robustness, \textit{IPOPT} does not exploit the temporal structure in OCPs, which makes it slow and limits its application in robotics. \rev{Recently, \textit{fatrop} \cite{fatrop} was proposed to mimic \textit{ipopt} while exploiting the KKT block-sparsity of OCPs. However, it is restricted to explicit dynamics and lacks effective parallelism. A summary of the aforementioned solvers is shown in Table. \ref{tab:solveattr}.}
\begin{table}[t]
    \centering
    \rev{
    \caption{Solver Attributes}
    \begin{minipage}{\textwidth}
    \begin{tabular}{lp{30pt}ccc}
        \hline
         \textbf{Solver} & \textbf{Parallelized} & \textbf{Eq.}\footnote{\rev{Eq./Ineq. - Equality/Inequality constraint handling.}} & \textbf{Ineq.} & \textbf{Dynamics} \\
         \hline
         \textit{Hippo} (ours) & yes & proj.\footnote{\rev{projection-based handling.}}/IPM & IPM & impl. \\
         \textit{fatrop} \cite{fatrop} & no & proj. & IPM & expl. \\
         \textit{acados} \cite{Verschueren2021acados} & no & IPM & IPM & impl. \\
         \textit{aligator} \cite{proxddp,parallelandproximalLQR} & yes & ALM & ALM & impl. \\
         \textit{mim\_solver} \cite{stagewisesqp} & yes & ALM & ALM & expl.\\
         \hline
    \end{tabular}
    \label{tab:solveattr}
    \vspace{-1em}
    \end{minipage}}
    \vspace{-2em}
\end{table}
\subsection{Main Contribution}
In this paper, we propose a simplistic solver implementation of regularized IPM and projection-based constrained SQP, namely \textit{Hippo}. \rev{It supports generic OCP formulations including implicit dynamics and cross-timestep pure-state constraints.} Similar to \textit{IPOPT}, at each SQP iteration, \textit{Hippo} directly solves the approximated QP for the primal-dual Newton step and handles hard equality constraints through either projection or IPM. Inspired by \cite{schwan2023piqp,ipopt_adaptive_barrier}, we use the MP-C scheme with a safeguard on the trial complementarity residuals.  We show that \textit{Hippo} can achieve robust convergence for difficult trajectory optimization problems. Furthermore, with a minimal number of tuning parameters and simple yet effective globalization, \textit{Hippo} does not require careful hand tuning of solver parameters. Our systematic comparison with state-of-the-art solvers demonstrates the superiority of \textit{Hippo} both in terms of convergence speed and number of problems it can solve.
\section{Constrained Optimal Control}\label{sec:eq_ocp}
\textbf{Notations}: Throughout this paper, bold characters denote vector/matrix values and normal characters denote scalars. We use $F_\*a$ to represent $\frac{\partial F}{\partial \*a}$, and $F_\*{ab}$ to represent $\frac{\partial^2F}{\partial \*a\partial \*b}$. $\delta \*a$ denotes a change in $\*a$. $(\cdot)$ is used as a placeholder for symbols.

In this paper, we aim to solve the following \rev{generic} OCP for a given initial state $\*x_0$
\begin{subequations}\label{eq:ocp}
\vspace{-0.7em}
	\begin{align}
		\min_{\tiny{\begin{matrix}\*{x}_{1\cdots N}\\\*u_{0\cdots N-1}\end{matrix}}}~~l_N(\mathbf{x}_N)&+\sum_{k=0}^{N-1}l_k(\mathbf{x}_k,\mathbf{u}_k),	\\
		\text{s.t.}~~~\forall k\in[0, N-1]:&~~\mathbf{f}(\mathbf{x}_k,\mathbf{u}_k,\mathbf{x}_{k+1})=\mathbf{0},\label{eq:impdynamics}\\
		&~~\mathbf{c}(\*x_k,\*u_k)=\*0,\label{eq:stateinputconstr}\\
		&~~\boldsymbol{\psi}(\mathbf{x}_k,\mathbf{u}_k)\preceq\mathbf{0},\label{eq:stateinputineq}\\
				&~~\*s(\*x_k, \*x_{k+1})=\*0,\label{eq:stateconstr}\\
		&~~\boldsymbol{\phi}(\*x_k,\*x_{k+1})\preceq\*0,\label{eq:stateineq}
	\end{align}
\end{subequations}
where $N$ is the number of shooting nodes, $\*x_k,\*u_k$ are the state and input of the $k$th stage, $\*x_N$ is the terminal state, $l_N, l_k$ are terminal and state-input running costs,  $\*f, \*c \text{ and } \*s$ specify implicit discrete dynamics (suitable for general integrators), stacked state-input and state-only constraints, respectively. $\b\psi$ denotes the state-input inequality constraints, and $\b\phi$ denotes state-only inequality constraints. 

To simplify the notation, we omit the subscript $k$ in the rest of the paper, and use a triplet $(\*x,\*u,\*y)$ to represent the current state, input, and next state of the $k$th stage, i.e., $\*y_{k}\equiv\*x_{k+1}$ for $k\in[0,N-1]$ and $\*y_{N-1}\equiv\*x_N$. In addition, the following stagewise notations will be used unless specified:
\begin{equation}\nonumber
\vspace{-1em}
\begin{array}{ll}
    \*h=[\*f^\top,\*s^\top,\*c^\top]^\top &\text{stacked equality constraints}\\
     \boldsymbol{\lambda}=[\boldsymbol{\lambda}^\top_f,\boldsymbol{\lambda}^\top_c,\boldsymbol{\lambda}^\top_s]^\top &  \text{multipliers of $\*h$}\\
     \*w = [\*x^\top, \*u^\top, \*y^\top]^\top & \text{stacked primal variables}\\
     \*g=[\b\phi^\top,\b\psi^\top]^\top & \text{stacked inequality constraints} \\
     \b\nu=[\b\nu_{\b\phi}^\top,\b\nu_{\b\psi}^\top]^\top & \text{multipliers of $\*g$}
\end{array}
\vspace{0.9em}
\end{equation}

In this section, we introduce the efficient projection-based Riccati recursion to solve the equality-constrained OCP without inequalities. In Sec. \ref{sec:ipm} we then explain how inequalities are handled with IPM.
The overall workflow of the projection-based recursion is shown in Algorithm. \ref{algo:single_iter_sqp}.
\subsection{Riccati Recursion}
\subsubsection{Intermediate stages} Bellman's principle of optimality defines a value function $V$ for each stage $k\in[0,N-1]$:
\begin{equation}\label{eq:bellmanopt}
		V(\*{x})=\min_{\*u,\*y}~l(\*x,\*u)+V(\*{y}),~ \text{s.t.}~\*h(\*w)=\*0
\end{equation}
which itself can be viewed as a constrained optimization problem. 
Similar to previous works \cite{proxddp}, we introduce the stagewise Lagrangian $Q$-function of \eqref{eq:bellmanopt} as
\begin{equation}\label{eq:Qfunc}
	\begin{aligned}
	Q&(\*w,\boldsymbol{\lambda})=l(\mathbf{x},\mathbf{u})+V(\*y)+\b\lambda^\top\*h(\*w)\\
	\end{aligned}
\end{equation}
In this work, we adopt a Gauss-Newton-like LQ approximation of \eqref{eq:Qfunc} as \cite{proxddp}, i.e., 1) all constraints are linearly approximated, and 2) the cost and value function $V$ are quadratically approximated. The corresponding KKT conditions can be written in the following recursive form
\begin{subequations}\label{eq:originalkktsys}%
	\begin{gather}
		\mathcal{K}\delta \boldsymbol{\xi}=-{\*k}_0-{\*K}_0\delta \*x,\label{eq:sparsekktmatrix}\\
    \begingroup
    \setlength{\arraycolsep}{2pt}
        \mathcal{K}=\begin{bmatrix}
		{Q}_\*{uu}&\rev{\*0}&\*{h}_\*u^\top\\
		\rev{\*0}&{Q}_\*{yy}&\*h_\*y^\top\\
		\*h_\*u&\*h_\*y&\rev{\*0}
	\end{bmatrix},\endgroup
		\*k_0=\begin{bmatrix}
		    Q_{\*u}^\top\\Q_{\*y}^\top\\\*h
		\end{bmatrix},
        \*K_0=\begin{bmatrix}
		    Q_{\*{ux}}\\Q_{\*{yx}}\\\*h_\*x
		\end{bmatrix}.
	\end{gather}
\end{subequations}
where $\boldsymbol{\xi}=[\*u^\top,\*y^\top,\boldsymbol{\lambda}^\top]^\top$. By solving \eqref{eq:originalkktsys}, its solution is
\begin{equation}
    \delta\b\xi=\underbrace{-\mathcal{K}^{-1}\*k_0}_{\*k_{\b\xi}}\underbrace{-\mathcal{K}^{-1}\*K_0}_{\*K_{\b\xi}}\delta\*x
\end{equation}
where $\*k_{\b\xi}$ and $\*K_{\b\xi}$ are respectively the zero and first order sensitivity of $\delta \b\xi$. Specially, $\*K_{\b\xi}=\delta\b\xi/\delta\*x\equiv\partial {\b\xi}/\partial\*x$. Note that in the later text, $\*k/\*K_{(\cdot)}, (\cdot)\neq 0$ refer to zero/first-order sensitivities of $(\cdot)$.
\subsubsection{Terminal node}\label{sec:recursion_rough_intro}
The value function of the terminal state $\*x_N$ is the terminal cost, i.e., $V_N(\*x_N)\equiv l_N(\*x_N)$. Since $\*y_k\equiv\*x_{k+1}$, it can be seen that $V_\*{yy}^k\equiv V_\*{xx}^{k+1}, V_\*{yy}^{N-1}\equiv V_\*{xx}^N$ and so is the Jacobian of $V$. The Riccati recursion can then be performed backward: starting from the terminal node $V_\*N$ derivatives, \eqref{eq:originalkktsys} is solved sequentially for $\delta\b\xi$ sensitivities; $V(\*y)$ derivatives of the previous stage are updated by the chain rule of differentiation
\begin{equation}\label{eq:rawvaluefuncderivative}
    \hspace{-1em}
    \begin{aligned}
        &V_\*{y}^{k-1}=\rev{V_\*{x}^k}= Q_{\*{x}}^k+(\*Q_{\b\xi}^\top\*K_{\b\xi})^k=Q_{\*{x}}^k+(\*k_0^\top\*K_{\b\xi})^k\\
        &V_\*{yy}^{k-1}=\rev{V_\*{xx}^k}= Q_{\*{xx}}^k+(\*Q_{\b\xi\*x}^\top\*K_{\b\xi})^k=Q_{\*{xx}}^k+(\*K_0^\top\*K_{\b\xi})^k
    \end{aligned}
\end{equation}

\subsection{Projection-based Factorization}\label{sec:nullspacesolvekkt}
By the definitions of $\*h$ and $\b\lambda$, the KKT system \eqref{eq:originalkktsys} can be expanded into
\begin{equation}\label{eq:sparsekktlinsys}
	\begin{bmatrix}
		Q_\*{uu}&\rev{\*0}&\*f_\*{u}^\top&\rev{\*0}&\*c^\top_\*u\\
		\rev{\*0}&Q_\*{yy}&\*f_\*y^\top&\*s_\*y^\top&\rev{\*0}\\
		\*f_\*u&\*f_\*y&\rev{\*0}&\rev{\*0}&\rev{\*0}\\
		\rev{\*0}&\*s_\*y&\rev{\*0}&\rev{\*0}&\rev{\*0}\\
		\*c_\*u&\rev{\*0}&\rev{\*0}&\rev{\*0}&\rev{\*0}
	\end{bmatrix}
	\begin{bmatrix}
		\delta\*u\\\delta\*y\\\delta\boldsymbol{\lambda}_f\\\delta\boldsymbol{\lambda}_s\\\delta\boldsymbol{\lambda}_c
	\end{bmatrix}=-
	\begin{bmatrix}
			\*u_0\\\*y_0\\\*f_0\\\*s_0\\\*c_0
	\end{bmatrix}.
\end{equation}
where the RHS are generic terms in $\*k_0$ or $\*K_0$. Null-space projection can be used to solve the above saddle-point problem. \rev{The invertibility of $\*f_\*y$ holds for common integrator dynamics in robotics, such as generic explicit/implicit Euler/RK4 integrators. Therefore, assuming invertible $\*f_\*y$, }the following can then be introduced to facilitate later factorization of \eqref{eq:sparsekktlinsys}:
\begin{equation}
\*F_0:=\*f_\*y^{-1}\*f_0,\*F_\*u:=\*f_\*y^{-1}\*f_\*u,\label{eq:proj_dyn_derivative}
\end{equation} 
whereby the nullspace basis $[\*Z_\*u^\top,\*Z_\*y^\top]^\top$ of the LHS equality constraint Jacobians $\*h_{\*u,\*y}:=[\*h_{\*u},\*h_{\*y}]$ (the $3\times2$ bottom left corner of \eqref{eq:sparsekktlinsys}) can be derived as
\begin{equation}\label{eq:nullspace_kernels}
    \rev{
    \*Z_\*u=\mathcal{N}\bigg(\begin{bmatrix}
        -\*s_\*y\*F_\*u\\
        \*c_\*u
    \end{bmatrix}\bigg), \*Z_\*y=-\*F_\*u\*Z_\*u.
    }
\end{equation}
where $\*Z_\*u, \*Z_\*y$ are respectively the \rev{nullspace basis $\mathcal{N}$} of the constraint Jacobians corresponding to $\*u,\*y$. We can then split the primal solution $\delta\*u,\delta\*y$ to \eqref{eq:sparsekktlinsys} by
\begin{equation}\label{eq:nullspace_decomposed_step}
    \begin{bmatrix}
        \delta\*u\\\delta\*y
    \end{bmatrix}=\begin{bmatrix}
        \*Z_\*u \\ \*Z_\*y
    \end{bmatrix} \delta \*z - \begin{bmatrix}
        \delta\hat{\*u}\\\delta\hat{\*y}
    \end{bmatrix}
\end{equation}
where $\delta \*z$ is the $(\*u,\*y)$ solution to \eqref{eq:sparsekktlinsys} in the nullspace of $\*h$ and \rev{$[\delta \hat{\*u}^\top,\delta \hat{\*y}^\top]^\top: \*h_{\*u,\*y}[\delta \hat{\*u}^\top,\delta \hat{\*y}^\top]^\top=\*h_0$ is the pseudoinverse} solution of the KKT primal feasibility condition. Similar to \cite{nullspacecrocoddyl}, by substituting \eqref{eq:nullspace_decomposed_step} into \eqref{eq:sparsekktlinsys} and left-multiplying the result by $[\*Z_u^\top,\*Z_y^\top]$, \eqref{eq:sparsekktlinsys} can be projected into the nullspace of $\*h$ as
\begin{subequations}\label{eq:nullspace_kkt}
    \begin{align}
        Q_\*{zz}\delta \*z&~=-\*z_0, \rev{where}\\
        Q_\*{zz}&:=\*Z_\*u^\top Q_\*{uu}\*Z_\*u+\*Z_\*y^\top Q_\*{yy}\*Z_y\\
        \*z_0&:=\*Z_\*u^\top\bar{\*u}_0+\*Z_\*y^\top\bar{\*y}_0\\
        \bar{\*u}_0&:=\*u_0-Q_\*{uu}\delta \hat{\*u}\\
        \bar{\*y}_0&:=\*y_0-Q_\*{yy}\delta \hat{\*y}
    \end{align}
\end{subequations}
This reduced system is applicable to various constrained dynamics, particularly those involving contact dynamics. When no $\*s,\*c$ is involved, $\*Z_\*u\equiv\*I$ and $\*Z_\*y\equiv-\*F_\*u$, $\delta\hat{(\cdot)}=0$, constituting classical Riccati recursion such as in \cite{mastalli20crocoddyl}.

By solving \eqref{eq:nullspace_kkt} for $\delta\*z$, we can recover the full primal step by \eqref{eq:nullspace_decomposed_step} and the dual step by \rev{solving of the KKT stationarity condition} after substituting in $\delta \*u,\delta\*y$:
\begin{equation}\label{eq:nullspace_solve_lambda}
    \rev{\*h_{\*u,\*y}^\top\delta \b\lambda=-\begin{bmatrix}
        \*u_0+Q_\*{uu}\delta\*u\\\*y_0+Q_\*{yy}\delta \*y
    \end{bmatrix}.}
\end{equation}
\begin{remark}\label{rem:rank_def_constr_jac} (Rank deficient constraint Jacobians) 
    We use Eigen \cite{eigenweb} full pivoting LU factorization to compute \eqref{eq:nullspace_kernels} and the pseudoinverse solutions of the primal and dual steps. It features a threshold-based decision of nonzero pivots; thus, it can be relatively robust against poor numerical conditions. In practice, especially with contacts, a specially regularized factorizer such as the one in \cite{carpentier2021proximal} can be employed to obtain bounded approximate solutions and nullspace projectors. When the constraint Jacobians are highly rank-deficient, regularization might be insufficient to get valid Newton steps, and the solver can converge to infeasible stationary points. In such cases, globalization strategies should be employed to escape such situations \cite{ipopt_implementation}.
\end{remark}
\begin{remark} (Recursive nullspace projection)
    The derivation of the nullspace basis in \eqref{eq:nullspace_kernels} is inspired by the classical hierarchical QP-based whole body controllers \cite{hqp2014}. It can be further extended if parts of the projected $\bar{\*h}_u$ are invertible, such as the case of lifted inverse dynamics with contacts, where one can use the inverse of the Delassus matrix to further derive an analytical nullspace projector.
\end{remark}
\begin{algorithm}[h]
\caption{Single Iteration of SQP\label{algo:single_iter_sqp}}
\SetKwBlock{DoParallel}{in parallel for each node}{end}

\KwData{Initial trajectory of $\*x,\*u$, cost functions $l_k, l_N$, dynamics $\*f$, constraints $\*c, \*s, \boldsymbol{\psi}, \boldsymbol{\phi}$ and their multipliers}

\DoParallel(\Comment{Pre-solving steps}){
    Update LQ approximation of the OCP.
    
    Nullspace factorization to compute \eqref{eq:nullspace_kernels}, $\delta \hat{(\cdot)}$ in \eqref{eq:nullspace_decomposed_step} and the precomputable parts of \eqref{eq:nullspace_kkt},\eqref{eq:nullspace_value_func_prop}. 
    
    Jacobian/Hessian modification by IPM \eqref{eq:ipm_q_derivative_increment}.
}
\For(\Comment{Backward recursion}){$k=N-1$ \KwTo $0$}{
    $V_{\*{y}}\leftarrow V_{\*{x}}^{k+1}$, $V_{\*{yy}}\leftarrow V_{\*{xx}}^{k+1}$.
    
    $Q_\*{y}\mathrel{+}= V_{\*{y}},Q_\*{yy}\mathrel{+}=V_{\*{yy}}$.
    
    Update $Q_\*{zz}$ and $\bar{\*y}_0$ for $\*k_0,\*K_0$ respectively.
    
    Solve nullspace KKT system \eqref{eq:nullspace_kkt} for $\*K_\*z$.
    
    Update $V_\*{xx}^k$, $V_\*{x}^k$ using \eqref{eq:nullspace_value_func_prop}.}

\DoParallel(\Comment{Post-solving steps}){
Compute $\*k_\*z, \*k_\*y,\*K_\*y$.}

\For{$k=0$ \KwTo $N-1$}{
$\delta\*x_{k+1}\leftarrow\delta\*y_k\mathrel=\*k_\*y^k+\*K_\*y^k\delta\*x^k$.}
\DoParallel(\Comment{Post-rollout steps}){
$\delta \*z=\*k_\*z+\*K_\*z\delta\*x$.

$\delta \*u=\*Z_\*u\delta\*z-\delta\hat{\*u}$.

Update $\b\lambda$ (if necessary).
}
\vspace{-0em}
\end{algorithm}
\subsection{Value Function Derivative Propagation}
Naively updating $V$ derivatives by \eqref{eq:rawvaluefuncderivative} requires computing full primal and dual first-order sensitivities, leading to high computational cost in the recursion stage. However, the residual-sensitivity-products (RSPs) $\*k_0^\top\*K_{\boldsymbol{\xi}}, \*K_0^\top\*K_{\boldsymbol{\xi}}$ can be transformed into a more computationally efficient form by exploiting the projection. 

Let superscript $l$ denote the quantities associated with the LHS of an RSP and $r$ denote those associated with $\*K_{\b\xi}$. For example, $\*u_0^l$ is $Q_\*u^\top$ in $\*k_0$ for $V_\*x$; $\delta \hat{\*u}^r, \delta \hat{\*y}^r$ are the \rev{pseudoinverse} solutions with RHS being $\*K_0$. By \eqref{eq:nullspace_solve_lambda} and $\*h_{\*u,\*y}[\delta \hat{\*u}^\top,\delta \hat{\*y}^\top]^\top=\*h_0$, the dual part of an RSP can be rewritten as
\begin{equation}
    \begin{aligned}
    \*h_0^{l,\top}\delta\b\lambda^r
    =&-\begin{bmatrix}
        \delta\hat{\*u}^{l}\\\delta\hat{\*y}^{l}
    \end{bmatrix}^\top\begin{bmatrix}
        \*u_0^r+Q_\*{uu}\delta\*u^r\\\*y_0^r+Q_\*{yy}\delta \*y^r
    \end{bmatrix}
    \end{aligned}
\end{equation}
By merging it with the primal part of the RSP, and splitting $\delta\*u^r,\delta\*y^r$ into the nullspace and \rev{pseudoinverse} steps by \eqref{eq:nullspace_decomposed_step}, we have that \rev{for the current stage}, $\forall (\cdot) \in \{\*x, \*{xx}\}$,
\begin{equation}
    \begin{aligned}\label{eq:nullspace_value_func_prop}
        V_{(\cdot)}=Q_{(\cdot)}+\*z_0^{l,\top}\delta\*z^r
        &-\bar{\*u}_0^{l,\top}\delta\hat{\*u}^{r}-\bar{\*y}_0^{l,\top}\delta\hat{\*y}^{r}\\
        &-\delta\hat{\*u}^{l,\top}\*u_0^r-\delta\hat{\*y}^{l,\top}\*y_0^r.
    \end{aligned}
\end{equation}
\rev{which is the $V_{\*y/\*{yy}}$ of the previous stage as in \eqref{eq:rawvaluefuncderivative}.} In this way, we only need to compute the sensitivities in the nullspace during the recursion.
Note that the parts unrelated to \rev{value function derivatives} can be precomputed in parallel.

\subsection{Linear Forward Rollout}
To construct the full Newton step of all decision variables, linear forward rollouts are performed initially for $\delta \*y$ using $\*k_\*y,\*K_\*y$. The rest primal and dual steps are later constructed in parallel using the updated $\delta \*x$ (from $\delta \*y$). Note that the nonlinear rollout widely used in DDPs is not applied here since a linear rollout reconstructs the true Newton step and enjoys more solid theoretical convergence guarantees \cite{stagewisesqp}.
\section{Sparsity-Preserving Constraint Handling}\label{sec:ipm}
In this section, we introduce a sparsity-preserving parallelizable regularized IPM implementation for both inequality constraints and rank-deficient equality constraints.
\subsection{Regularized Primal-Dual Interior Point Method}\label{sec:primal_dual_ipm}
Extending the equality constrained OCP to the full problem in \eqref{eq:ocp} with inequality constraints $\*g$ results in a new Lagrangian $\tilde{Q}$ and its corresponding LQ-approximated KKT conditions
\begin{subequations}\label{eq:ineq_new_kkt}
\vspace{-0.7em}
    \begin{align}
        \tilde{Q}(\*w,\b\lambda,\b\nu):=Q(\*w,\b\lambda)+\b\nu^\top\*g(\*w),\label{eq:ineq_lagrangian_q_func}\\
        \begin{bmatrix}
            \tilde{\mathcal{K}}&\*g_{\b\xi}^\top\\
            \*g_{\b\xi}&\rev{\*0}
        \end{bmatrix}\begin{bmatrix}
            \delta\b\xi \\ \delta\b\nu
        \end{bmatrix}=-\begin{bmatrix}
            \tilde{\*k}_0\\\*g
        \end{bmatrix}-\begin{bmatrix}
            \tilde{\*K}_0\\\*g_\*x
        \end{bmatrix},\\
        \*h+\*h_\*w\delta\*w=\*0,\\
        \*g+\*g_\*w\delta\*w\leq\*0,\label{eq:ineq_residual_feasibility}\\
        \b\nu\odot\*g=\*0,\label{eq:comp_slack_origin}
    \end{align}
\end{subequations}
where $\b\nu$ is the stacked inequality multiplier of $\*g$, $\tilde{\mathcal{K}}, \tilde{\*k}_0, \tilde{\*K}_0$ are defined similarly to \eqref{eq:originalkktsys} but w.r.t. $\tilde{Q}$, $\odot$ denotes element-wise product. Note that $\*g_{\b\lambda}=\*0$.
\eqref{eq:ineq_new_kkt} can be transformed into an equality-constrained problem via the primal-dual interior point method, by introducing slack variables $\*t$ into \eqref{eq:ineq_residual_feasibility} and relaxing the complementary slackness \eqref{eq:comp_slack_origin} with a barrier parameter $\mu$, i.e., the \textit{perturbed KKT condition}
\begin{subequations}\label{eq:perturbedKKT}
\vspace{-0.5em}
	\begin{align}
		&\*r_g:=\*g+\*t,&&\*r_s:=\boldsymbol{\nu}\odot\*t-\mu{\*1},&&\label{eq:ipm_res}\\
		&\*r_g+\*g_\*w\delta\*w=\*0,&&\*r_s=\*0,&&\boldsymbol{\nu},\*t\succeq\*0,
	\end{align}
\end{subequations}
which is a bilinear system due to $\b{\nu}\odot\*t$. Linearizing \eqref{eq:perturbedKKT} w.r.t. $\b\nu,\*t$ results in the following new KKT system
\begin{equation}\label{eq:ipmkktlinearequation}
	\begin{aligned}
		&\begin{bmatrix}
			\tilde{\mathcal{K}}&\*g_{\boldsymbol{\xi}}\top&\rev{\*0}\\
			\*g_{\boldsymbol{\xi}}&-\rho\*I&\*I\\
			\rev{\*0}&\*T&\*N
		\end{bmatrix}\begin{bmatrix}\delta{\boldsymbol{\xi}}\\\delta\boldsymbol{\nu}\\\delta\*t\end{bmatrix}=-\begin{bmatrix}
			\tilde{\*k}_0\\\*r_g\\\*r_s
		\end{bmatrix}-\begin{bmatrix}
			\tilde{\*K}_0\\
			\*g_{\*x}\\
			\*0
		\end{bmatrix}\delta \*x,
	\end{aligned}
\end{equation}
where $\*T=\text{diag}(\*t)$, $\*N=\text{diag}(\boldsymbol{\nu})$, $\rho$ is the multiplier regularization parameter to avoid ill-conditioning when $\*t\to 0$. We can then exploit the structure of \eqref{eq:ipmkktlinearequation} to split the solution into a $\b\xi$-step and a $(\b\nu,\*t)$-step:
\begin{subequations}\label{eq:ipm_split}
\vspace{-0.5em}
    \begin{align}
        \delta \b\xi&=\hat{\tilde{\mathcal{K}}}^{-1}(-\hat{\tilde{\*k}}_0 - \hat{\tilde{\*K}}_0\delta \*x)\\
        \delta \b\nu&=\*T_\rho^{-1}[-\*r_s - \*N(\*r_g+\*g_{\*w}\delta \*w)]\\
        \delta \*t&=-\*r_g-\*g_{\*w}\delta \*w+\rho\delta\b\nu
    \end{align}
\end{subequations}
where $\*T_\rho:=\*T+\rho\*N$, $\hat{\tilde{(\cdot)}}$ denotes modified quantities due to the modification of the $\tilde{Q}$ Jacobian and Hessian, which is represented as incremental values ($\mathcal{W}:=\{\*x,\*u,\*y\}$):
\begin{subequations}\label{eq:ipm_q_derivative_increment}
	\begin{align}
    &\forall {(\cdot)}\in \mathcal{W}:
        \Delta \tilde{Q}_{(\cdot)}=[\*T_\rho^{-1}\*N\*g+\mu\*T_\rho^{-1}]^\top\*g_{(\cdot)},\label{eq:jacobian_mod_ipm}\\
     &       \forall {(\cdot)}\in \mathcal{W}\times \mathcal{W}:
		\Delta \tilde{Q}_{{(\cdot)}}=\*{g}_{(\cdot)}^\top\*T_\rho^{-1}\*N\*g_{(\cdot)}.
	\end{align}
\end{subequations}
Due to the diagonal nature of $\*T_\rho^{-1}\*N$ and the definition of $\*g$, ${Q}_{\*{uy}}$ will still be zero. Therefore, the sparsity in \eqref{eq:sparsekktlinsys} is preserved. With the splitting in \eqref{eq:ipm_split}, $(\b\nu,\*t)$-steps of different shooting nodes can be computed in parallel after the sequential forward rollout instead of solving the whole KKT system \eqref{eq:ipmkktlinearequation} directly.
\begin{theorem}\label{th:ipm_alm}
Assuming non-negativeness of $\*t,\b\nu$, the regularization term guarantees bounded Hessian modification \eqref{eq:ipm_q_derivative_increment}. Furthermore, assuming $\*r_s\approx\*0$, the Jacobian modification $\eqref{eq:jacobian_mod_ipm}$ will not explode numerically.
\end{theorem}

\begin{proof} The boundedness can be deduced by the non-negativeness of $\*t, \b\nu$:
\begin{equation}
\*T_\rho^{-1}\*N=(\*T\*N^{-1}+\rho\*I)^{-1}\prec\rho^{-1}\*I
\end{equation}
\rev{where $\prec$ denotes the element-wise less-equal}. By $\*r_s\approx\*0$, we have $\*T\*N\approx\mu\*I$. The magnitude of the Jacobian modification \eqref{eq:jacobian_mod_ipm} is bounded by
\begin{equation}
    ||\*T_\rho^{-1}\*N\*g+\mu\*T_\rho^{-1}||\approx||\*T_\rho^{-1}\*N(\*g+\*t)||\leq\rho^{-1}||\*g+\*t||.
\end{equation}
\end{proof}

The Hessian boundedness avoids numerical ill-posedness of the original naive IPM. It greatly improves \rev{numerical} robustness with only a very low $\rho$ (by default 1e-8 in our setting). Theorem \ref{th:ipm_alm} shows that when inequality constraints become active, the convergence characteristics of IPM will be similar to vanilla penalty-based equality constraint handling.

\subsection{Adaptive Barrier Strategy}
We implemented the MP-C variant from \cite{schwan2023piqp} to adaptively adjust the barrier parameter $\mu$. In the predictor step, \eqref{eq:sparsekktlinsys} is factorized and the primal steps and IPM variables are computed with $\mu=0$. The new $\mu$ is then computed as 
\begin{equation}\label{eq:mp-c_update_mu}
    \sigma=\text{clip}\bigg(\frac{\sum (\b\nu+\delta\b\nu)^\top(\*t+\delta \*t)}{\sum\b\nu^\top\*t},0,1\bigg)^3, \mu\leftarrow\sigma \frac{\sum \b\nu^\top\*t}{n_\text{ipm}},
\end{equation}
where $\sigma$ is the recentering parameter bounded within $[0,1]$ and $n_\text{ipm}$ is the total dimension of IPM constraints. In the corrector step, the Lagrangian Jacobian is modified with the updated $\mu$ and a corrector term $\delta\b\nu\odot\delta\*t$. The Riccati recursion is performed again to update the value function Jacobians $V_\*x$ with the previous KKT factorization from the predictor step.

Note that in both predictor and corrector steps, a step size upper bound must be computed to ensure the slackness and multipliers are positive:
\begin{subequations}\label{eq:ipm_ls}
    \begin{align}
        \alpha^{\max}&=\max\{\alpha\in[0,1]|\*t + \alpha\delta \*t \ge (1-\tau)\*t\},\\
        \alpha_{\b\nu}^{\max}&=\max\{\alpha_{\b\nu}\in[0,1]|\b\nu + \alpha_{\b\nu}\delta \b\nu \ge (1-\tau)\b\nu\}, 
    \end{align}
\end{subequations}
where $\alpha^{\max}, \alpha_{\b\nu}^{\max}$ are respectively primal (including $\*t$) and inequality dual step size bounds ($\tau=0.995$ in our setting). We also implemented a safeguarding mechanism that will only accept the corrector term if the trial complementarity does not increase \cite{ipopt_adaptive_barrier}.
\subsection{Iterative Refinement}
When the SQP with IPM is close to convergence, its poor numerical condition may cause relatively high error in the stationary condition and slow down convergence to high-accuracy solutions. If necessary, iterative refinement can be employed to reduce the error by setting the RHS $\*k_0$ to be the stationarity residual and performing the recursion similar to the corrector step.
\subsection{Fixed-size Backtracking Line Search}\label{sec:ipm_ls_globalization}
Similar to IPOPT \cite{ipopt_implementation}, our solver applies $\alpha\leq\alpha^{\max}$ to both primal and equality dual steps, and $\alpha_{\b\nu}\leq\alpha_{\b\nu}^{\max}$ only to inequality dual steps. A simple fixed-step-size backtracking line-search is employed, which decreases $\alpha$ by a fixed step $\Delta\alpha:=\alpha^{\max}/n_\text{ls}$ for a maximum number of trials $n_\text{ls}$ when no progress is made in both primal and dual residuals. If it cannot find any descent step, to escape from such local stationary point, the solver 1) scales $\alpha^{\max}$ from \eqref{eq:ipm_ls} to be no higher than a predefined minimal value $\alpha^{\min}$ and 2) resets $\mu=1$ and re-initialize the IPM slacks and multipliers. In practice, we \rev{found $\alpha^{\min}=$ $0.01$ a robust choice.}
\subsection{Equality Constraint Handling via IPM}\label{sec:eq_ipm}
When an equality constraint is potentially rank-deficient as introduced in Remark \ref{rem:rank_def_constr_jac}, projection-based handling might lead to severe difficulties for convergence. Similar to \textit{acados}, such equality constraints can be transformed into box constraints with zero boundaries, \rev{i.e., $\*h+\*t_p=\*0,\*t_n-\*h=\*0,\*t_p,\*t_n\succeq 0$, where $\*t_p,\*t_n$ represent the bilateral slack variables and they will be pushed to zero simultaneously by IPM to ensure feasibility.}

\subsection{Implementation}

The optimizer is implemented in C++ based on \textit{Eigen} \cite{eigenweb} and \textit{BLASFEO} \cite{Frison2019TheBlasfeo}. It also provides a code-generation pipeline for highly-optimized function derivative evaluation, based on \textit{CasADi} \cite{Andersson2019casadi} and \textit{Pinocchio} \cite{carpentier2019pinocchio}. 
\section{Benchmarks and Discussion}\label{sec:benchmark_n_discussion}
\begin{table}[t]
    \centering
    \vspace{1em}
    \caption{Comparison of Solvers and Their Characteristics}
\begin{minipage}{\textwidth}
    \begin{tabular}{lll}
        \hline
        \textbf{Solver} & \textbf{Inner Loop}\footnote{'lin' means only updating zero-order sensitivities. 'quad' means com-\\puting new Newton steps of the LQ subproblems.} & \textbf{Param Update Scheme} \\
        \hline
        Hippo\rev{/fatrop} & iterative refinement (lin) & per outer loop \\
        acados & hpipm qp loops (quad) & reset per outer loop \\
        aligator(p)\footnote{{the serial version of aligator does not have inner loop w.r.t. SQP steps}} & iterative refinement (lin) & if primal res. $>$ tol \\
        mim\_solver & ADMM step (lin) & per fixed num. updates \\
        \hline
    \end{tabular}
    \vspace{-1em}
    \end{minipage}
    \vspace{-2em}
    \label{tab:solver_comparison}
\end{table}
In this section, we benchmark \textit{Hippo} against prior works \rev{\textit{fatrop} \cite{fatrop} (commit 9a290fd
), \textit{acados} \cite{Verschueren2021acados} (commit 6111f01), \textit{aligator} \cite{proxddp,parallelandproximalLQR} (commit af2826), \textit{mim\_solver} \cite{stagewisesqp} (commit bd8852a)} in two representative examples to show the robustness of our solver: 
1) UR5 reaching to test its capability of handling rank-deficient hard constraints, and 2) Go2 locomotion to test its convergence under intermittent contacts in multi-contact settings commonly used in optimization-driven TAMP \cite{dhédin2025simultaneouscontactsequencepatch,taouil2025physicallyconsistenthumanoidlocomanipulation,ciebielski2025taskmotionplanninghumanoid}. Note that these are different from common MPC settings, which are usually simpler with good warm starts and do not require exact resolution (e.g. real-time iteration scheme that emphasizes timing per QP \cite{stagewisesqp,Verschueren2021acados}). For all experiments, we carefully tuned the parameters of \textit{aligator} and \textit{mim\_solver} such that they can solve as many problems as possible. Table. \ref{tab:solver_comparison} provides more details about each solver. \rev{A test is marked `Success` if the $L_\infty$ KKT residual (including constraint violation, stationarity, and complementarity) is below absolute tolerance.} To ensure fair comparison, all tested solvers used linear rollout and line search; \rev{forward whole-body dynamics were} used for problem formulation. For each test, the costs and initialization were unified for all solvers. The hardware platform is AMD Ryzen 9 7945HX. All solvers were compiled with \textit{-march=native}. \textit{Hippo(i)} represents tests with non-dynamics equality constraint handled as Sec. \ref{sec:eq_ipm}. \rev{Note that due to implementation limitations of \textit{fatrop}, its total wall timing reported later will be the sum of solver-only timing and $\frac{\text{function-evaluation timing}}{\text{number of threads}}$ for comparison. However, in practice, the wall time of \textit{fatrop} will be much higher due to the lack of parallelization of the solver and the inefficient evaluation of functions.}
\subsection{UR5 Random Reaching}\label{sec:ur5}
\begin{table}[t]
    \centering
    \vspace{2em}
    \caption{Solver Parameters for UR5 Reaching Benchmark}
    \begin{minipage}{\textwidth}
    \begin{tabular}{lcccc}
        \hline
        \textbf{Solver} & \textbf{Iters} & \textbf{Inner Iters} & \textbf{Inner Tol.} & \textbf{Param}\footnote{barrier parameter for IPM and penalty parameter for ALM/ADMM.} \\
        \hline
        Hippo & 100 & 2 & 1e-10 (abs) & 1 \\
        acados & 100 & 50 & 1e-3 (abs) & 1 \\
        fatrop & 100 & 5 & 1e-8 (abs) & 1 \\
        aligator(s) & 400 & 0 & varying & 0.1 \\
        \rev{mim\_solver(1)/(2)} & 400 & \rev{5(15)/20(60)}\footnote{(..) the number of fixed inner steps to update the penalty parameters.} & 1e-3 (abs) & 0.1\\
        \hline
    \end{tabular}
    \vspace{-1em}
    \end{minipage}
    \label{tab:solver_parameters_ur5}
\end{table}
\begin{figure}[t]
    \centering
    \includegraphics[width=\linewidth]{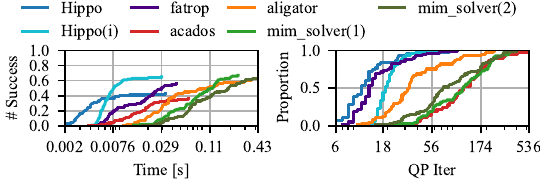}
    \vspace{-2em}
    \caption{\rev{UR5 random SE3 reaching benchmark results. Left plot is the number of sucess v.s. the wall time. Right plot shows the number of QP iterations distribution of the tests.}}
    \vspace{-1em}
    \label{fig:ur5_reach}
\end{figure}
In this experiment, we borrow the UR5 reaching task from \cite{proxddp}, i.e., solving for 100 random target end-effector $SE(3)$ placements in the hemisphere around the robot base frame origin:
\begin{equation}
    \*M^\text{des}\sim \rev{\mathcal{U}([-0.6,0.6]^2 \times [0, 0.6] \times SO(3))}.
\end{equation} 
where $\mathcal{U}$ denotes uniform distribution. Each target is paired with a random initial configuration uniformly sampled from $\mathcal{U}[-1, 1]$ as in \cite{proxddp}. The target is added as a hard constraint $\*m$ for the terminal state, i.e., 
\begin{equation}
\*m:=\log[\*M^\text{des}\ominus\*M_N]=\*0_6,
\end{equation}
where $\*M_N\in SE(3)$ is the terminal end-effector placement and $\ominus$ denotes the $SE(3)$ difference and $\log$ is the logarithmic map from $SE(3)$ to $\mathfrak{se}(3)$ . The trajectory horizon is $N = 50$, with a timestep of $20$ms. State and torque limits are included. Each test runs on 4 threads. The solver settings are shown in Table. \ref{tab:solver_parameters_ur5} and the result is shown in Fig. \ref{fig:ur5_reach}. The termination absolute tolerance is 1e-3. We used the serial version of \textit{aligator} (marked by \textit{(s)}), because we encountered numerical issues with its parallel version (denoted as \textit{aligator(p)}). \rev{For \textit{fatrop}, we use the compiled version of its code-generated \textit{CasADi OptiStack} interface to maximize its performance.} \rev{We used a high iteration limit from \cite{proxddp} for the ALM baselines since they often require more iterations to converge}. Two sets of parameters of \textit{mim\_solver} were used as recorded in Table \ref{tab:solver_parameters_ur5}. 
\begin{table}[t]
    \centering
    \vspace{2em}
    \caption{Statistics for UR5 example}
    \label{tab:ur5_reach}
    \begin{minipage}{\linewidth}
    \centering\rev{
    \begin{tabular}{m{1pt}>{\centering}m{18pt}>{\centering}m{23pt}>{\centering}m{25pt}>{\centering}m{20pt}>{\centering}m{20pt}c}
        \hline
        &\textbf{Hippo} &\textbf{Hippo(i)} & \textbf{fatrop} & \textbf{acados} & \textbf{aligator} & \textbf{mim\_solver} \\
        \hline
         \textbf{\#} &43&65&56 &36/81\footnote{\rev{Only 36 out of the 81 solutions have acceptable slack values ($\le$1e-3).}}&61&(1):67;(2):64\\
        \textbf{t}\footnote{\rev{`\#`: number of solved problems; `\textbf{t}`: average timing per QP (ms)}} & 0.33 & 0.32 & 0.86 & 0.13 & 1.02 & (1):0.61;(2):1.1 \\
        \hline
    \end{tabular}}
    \vspace{-1em}
    \end{minipage}
    \vspace{-2em}
\end{table}
As shown in Fig. \ref{fig:ur5_reach} and Table. \ref{tab:ur5_reach}, \textit{Hippo(i)} solved more problems than \textit{Hippo} due to the potential Jacobian rank-deficiency of $\*m$ but with more iterations. \rev{Benefiting from the \textit{ipopt} globalization scheme, \rev{\textit{fatrop}} solved more tests than \textit{Hippo} but still fewer than \textit{Hippo(i)}.} \textit{acados} only works with \rev{heavily penalized slacks to push the primal residual below the tolerance} and \rev{consequently prone to local infeasible minima.} \textit{aligator} and \textit{mim\_solver} solved similar numbers of problems to \textit{Hippo(i)} \rev{with more iterations}. This demonstrates that our regularized \rev{IPM-based equality-constraint handling} is more robust and efficient, \rev{despite requiring more iterations than pure projection, since} $\mu$ must be gradually driven to zero.

It is worth noting that \textit{Hippo} does not solve QP subproblems the fastest due to the generic nullspace projection. In contrast, \rev{\textit{fatrop} and the \textit{HPIPM} backbone of \textit{acados}} fully exploit the efficient \textit{BLASFEO}; \textit{aligator} and \textit{mim\_solver} utilize \textit{pinocchio} to directly eliminate the contact forces by computing its derivative w.r.t. the kinematic states. The speed advantage of \textit{hippo} is \rev{typically due to its lower number of QP iterations}, since it converges with fewer iterations and does not require exact minimization of subproblems. \rev{\textit{fatrop} also showed a similar speed advantage.} On the contrary, each SQP step of \textit{mim\_solver} and \textit{acados} can take a large number of inner iterations but make little progress, especially if the subproblem is already ill-conditioned, e.g., \textit{mim\_solver(1)/(2)} data in Fig. \ref{fig:ur5_reach}. As discussed in Remark \ref{rem:rank_def_constr_jac}, for ill-posed OCPs, escaping local infeasible points by taking primal steps via globalization is more effective than exact minimization of LQ subproblems. In our practice, we found that \rev{\textit{acados} and \textit{mim\_solver} are} highly sensitive to their inner-loop settings.

\subsection{Go2 2-step Locomotion}
\begin{table}[t]
    \centering
    \vspace{1em}
    \caption{\rev{ALM} Solver Parameters Go2 Locomotion Benchmark}
    \begin{minipage}{\linewidth}
    \centering\rev{
    \begin{tabular}{lcccc}
        \hline
        \textbf{Solver} & \textbf{SQP Iters} & \textbf{Inner Iters} & \textbf{Inner Tol.} & \textbf{Param}\footnote{barrier parameter for IPM and penalty parameter for ALM/ADMM} \\
        \hline
        aligator(p) & 200 & 5 & varying & 0.1 \\
        mim\_solver & 200 & 100(50)\footnote{(..) the number of fixed steps to update the penalty parameters} & 1e-3 (abs) &  0.1\\
        \hline
    \end{tabular}}
    \vspace{-1em}
    \end{minipage}
    \vspace{-2.5em}
    \label{tab:solver_configurations_go2}
\end{table}
\begin{figure*}
\vspace{1em}
    \centering
    \includegraphics[width=\linewidth]{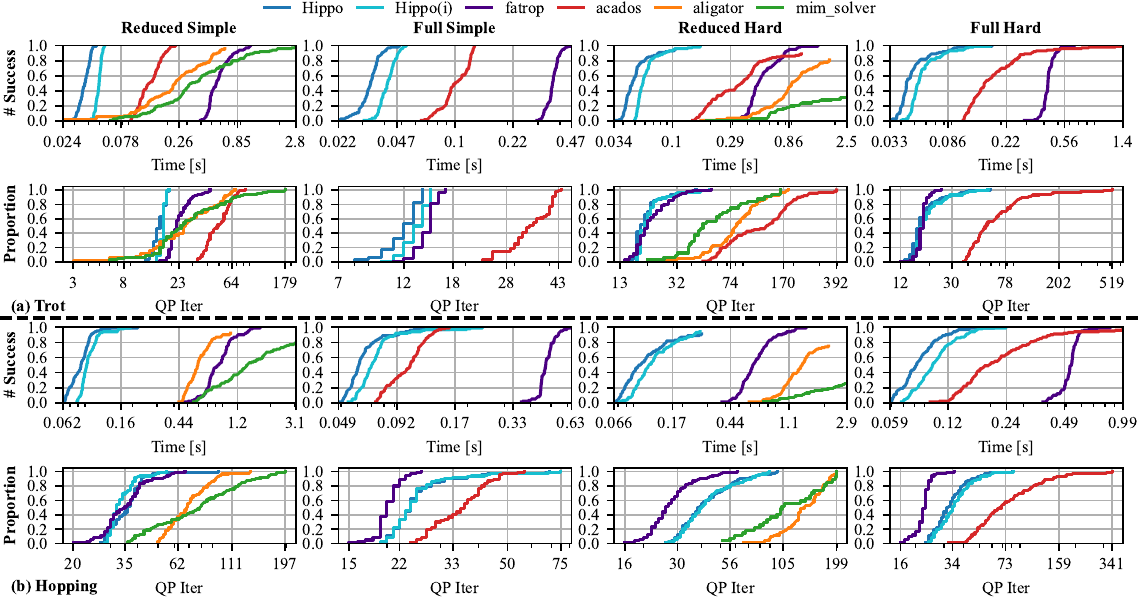}
    \vspace{-2em}
    \caption{\rev{Go2 2-step locomotion benchmark results. 100 random trials are tested for each setting. Missing solver data for a setting indicates that the solver failed all tests for that setting. }}
    \vspace{-1em}
    \label{fig:go2_all}
\end{figure*}
\begin{table}[t]
    \centering
    \caption{Statistics for Go2 example}
    \newcommand{\STAB}[1]{\begin{tabular}{@{}c@{}}#1\end{tabular}}
    \label{tab:go2_loco}
    \begin{minipage}{\linewidth}
    \centering\rev{
    \begin{tabular}{p{25pt}cccccccc}
        \hline
        & \multicolumn{2}{c}{\textbf{RS}\footnote{R: reduced; S: simple; F: full; H: hard. '\#' : number of solved problems. '\textbf{t}': average timing (ms) per QP. }} & \multicolumn{2}{c}{\textbf{FS}}& \multicolumn{2}{c}{\textbf{RH}} & \multicolumn{2}{c}{\textbf{FH}} \\
        \hline
        \textbf{trot}& \textbf{\#} & \textbf{t} & \textbf{\#} & \textbf{t} & \textbf{\#} & \textbf{t} & \textbf{\#} & \textbf{t}\\
        \hline
        \textit{Hippo}& 100        &2.35   & 100   &2.75   & 96   &2.55   & 100&2.67 \\
        \textit{Hippo(i)}& 100     &2.93   & 100   &3.13   & 99   &3.01   & 100&2.95 \\
        \textit{fatrop}& 100       &24.9   & 100   &25.0   & 100   &24.9   & 100&23.9 \\
        \textit{acados}& 100       &2.86   & 100   &2.99   & 89   &2.60   & 100&2.81 \\
        \textit{aligator}& 96     & 8.68  & 0     &-      & 81   &9.95   & 0&- \\
        \textit{mim\_solver}& 98  &13.37  & 0     &-      & 32   &14.96  & 0&- \\
        \hline
        \textbf{hopping}& \textbf{\#} & \textbf{t} & \textbf{\#} & \textbf{t} & \textbf{\#} & \textbf{t} & \textbf{\#} & \textbf{t}\\
        \hline
        \textit{Hippo}& 100        &2.35   & 100   &2.57   & 91   &2.46   & 100&2.57 \\
        \textit{Hippo(i)}& 100     &2.89   & 100   &2.93   & 94   &2.86   & 100&2.93 \\
        \textit{fatrop}& 100       &24.3   & 100   &23.6   & 100   &24.9   & 100&24.4 \\
        \textit{acados}& 0       &-      & 100   &2.97   & 0   &-      & 97&2.86 \\
        \textit{aligator}& 92     &8.58   & 0     &-      & 75   &8.97   & 0&- \\
        \textit{mim\_solver}& 79  &15.42  & 0     &-      & 27    &14.29  & 0&- \\
        \hline
    \end{tabular}
    \vspace{-1em}
    }
    \end{minipage}
    \vspace{-2em}
\end{table}
In this experiment, we formulate a generic quadruped locomotion problem where the robot moves towards two randomly sampled reference base x-y positions within a rectangular area centered around the robot initial position (i.e., the origin). Two gaits (trot and hopping) are tested, and each is added to the OCP as explicit phase contact constraints. The termination absolute tolerance is 1e-3. We provide two difficulty levels of the randomly sampled references. Let $\*r_1, \*r_2$ be the desired base position after the first and last step, respectively. The simple version samples $\*r_2$ from a small region, and then uses the middle point as $\*r_1$, i.e.,
\begin{equation}
    \*r_1=\frac {\*r_2}2, \*r_2\sim \mathcal{U}[-0.5,0.5]\times\mathcal{U}[-0.5,0.5].
\end{equation}
The \rev{hard} version samples from a larger area for two unrelated $\*r_1,\*r_2$, leading to larger movements, i.e.,
\begin{equation}
    \*r_1,\*r_2\sim \mathcal{U}[-1,1]\times\mathcal{U}[-1,1].
\end{equation}
The regularization cost for this example assigns higher weights to base states, emphasizing the tracking of position references. Joint and torque limits are included. No foot references are provided, similarly to \cite{taouil2025physicallyconsistenthumanoidlocomanipulation}. \rev{We tested two constraint settings: 1) "\textit{full}" with friction cone and position-level contact constraint $[v_{c,x}, v_{c,y}, v_{c, z}+100 z_c]^\top=\*0$ where $v_{c,(\cdot)}$ are the linear contact velocities and $z_c$ is the foot height; 2) "\textit{reduced}" without the friction cone and only the zero-contact-velocity constraints, while $z_c$ is penalized.} \rev{The ALM solver parameters are recorded in Table. \ref{tab:solver_configurations_go2}; the IPM ones are the same as Table. \ref{tab:solver_parameters_ur5}.} \rev{To our best effort, \textit{fatrop} is tested with direct \textit{CasADi OptiStack} interface since we were not able to compile the over-sized code-generated file.} The results for this scenario are shown in \rev{Fig. \ref{fig:go2_all}} and Table. \ref{tab:go2_loco}.

\rev{\textit{Hippo} and \textit{Hippo(i)} solved nearly all problems except for the \textit{reduced hard} setting.} \rev{The failures stemmed from the simplistic globalization scheme \ref{sec:ipm_ls_globalization} that results in small steps upon quasi-convergence. Conversely, \textit{fatrop} leveraged the systematic \textit{ipopt} globalization and solved all tests. \textit{acados} solved fewer cases and failed \textit{reduced} hopping tests completely. These observations indicate that reduced OCPs with penalties rather than constraints are prone to degeneracy and require systematic globalization. It explains the performance drop of \textit{Hippo} and the catastrophic failure of \textit{acados}. Moreover, the local minima of such OCPs could be infeasible, as seen in the \textit{acados} results from Table \ref{tab:ur5_reach} of Sec. \ref{sec:ur5}, and thus being less useful for constraint-based planning such as \cite{taouil2025physicallyconsistenthumanoidlocomanipulation, dhédin2025simultaneouscontactsequencepatch,huaijiang_montecarlo}. In strictly constrained \textit{full} settings, \textit{Hippo/Hippo(i)} demonstrated the fastest speed and comparable robustness to \textit{fatrop}, even in hopping tests where \textit{fatrop} took fewer QP iterations. This shows the effectiveness of the parallelism and globalization of our solver in handling constrained OCPs. In contrast, despite a higher iteration budget, \textit{aligator} and \textit{mim\_solver} could only solve parts of the \textit{reduced} tests and failed all \textit{full} settings. These results again corroborate the superior robustness and convergence of IPM over ALM as discussed in Sec. \ref{sec:relatedwork} and \ref{sec:ur5}}.

\section{Conclusion}\label{sec:conclusion}
In this paper, we introduced \textit{Hippo}, a minimalistic implementation of interior-point and projection-based SQP. \textit{Hippo} demonstrated robust convergence in the benchmark of hard examples against existing state-of-the-art solvers without a heavy workload on solver parameter tuning. As a generic constrained trajectory optimizer, \textit{Hippo} could be helpful for data generation for IL and optimization-driven TAMP. For future work, more comprehensive globalization strategies, such as the \textit{IPOPT} feasibility restoration phase and filter line search \cite{ipopt_implementation}, could be integrated into the solver framework to further enhance its robustness.

\bibliography{refs} 
\bibliographystyle{ieeetr}

\end{document}